\documentclass[letterpaper, 10 pt, journal, twoside]{IEEEtran}
\usepackage{microtype}
\usepackage{amsmath, amssymb}
\usepackage{graphicx}
\usepackage{stfloats}
\usepackage[ruled,vlined,noend]{algorithm2e}
\usepackage{booktabs, multirow}
\usepackage{balance}
\usepackage[colorlinks]{hyperref}
\usepackage{color}
\usepackage{float}
\usepackage{gensymb}
\usepackage[usenames,dvipsnames]{xcolor}
\usepackage{soul}

\usepackage[noadjust]{cite}

\usepackage{pifont}

\newtheorem{lem}{Lemma}
\newtheorem{theorem}{Theorem}
\newtheorem{remark}{Remark}
\newtheorem{definition}{Definition}
\newtheorem{proof}{Proof}

\ifCLASSINFOpdf
\else
\fi

\begin{document}
\title{Safe Autonomous Environmental Contact for Soft Robots using Control Barrier Functions}

\author{Akua K. Dickson$^{1*}$, Juan C. Pacheco Garcia$^{2*}$, Meredith L. Anderson$^{2}$, Ran Jing$^{2}$, Sarah Alizadeh-Shabdiz$^{2}$, Audrey X. Wang$^{2}$, Charles DeLorey$^{3}$, Zach J. Patterson$^{4}$, Andrew P. Sabelhaus$^{1,2}$%
\thanks{Manuscript received: May, 7, 2025; Revised August, 1, 2025; Accepted August, 30, 2025. This paper was recommended for publication by Editor Cecilia Laschi upon evaluation of the Associate Editor and Reviewers' comments. This work was in part supported by the U.S. National Science Foundation under Award No. 2340111, 2209783, and an NSF Graduate Research Fellowship. \textit{(Corresponding Author: Andrew P. Sabelhaus.)}}
\thanks{$^1$A.K. Dickson and A.P. Sabelhaus are with the Division of Systems Engineering, Boston University, Boston MA, USA. {\tt \footnotesize \{akuad, asabelha\}@bu.edu} }
\thanks{$^2$M.L. Anderson, J.C. Pacheco Garcia, R. Jing, S. Alizadeh-Shabdiz, A.X. Wang, and A.P. Sabelhaus are with the Department of Mechanical Engineering, Boston University, Boston MA, USA. {\tt \footnotesize \{merland, jcp29, rjing, sar3, axwang\}@bu.edu} }
\thanks{$^3$C. DeLorey is with the School of Electronic and Electrical Engineering, University of Leeds, Leeds, UK. ({\tt \footnotesize elchd@leeds.ac.uk})}%
\thanks{$^4$Z. Patterson is with the Department of Mechanical and Aerospace Engineering, Case Western Reserve University, Cleveland, OH, USA. ({\tt \footnotesize zpatt@case.edu})}%
\thanks{$^*$ Equal Contribution.}%
\thanks{Digital Object Identifier (DOI): see top of this page.}
}

\markboth{IEEE Robotics and Automation Letters. Preprint Version. Accepted August, 2025}
{Dickson \MakeLowercase{\textit{et al.}}: Safe Autonomous Environmental Contact for Soft Robots using Control Barrier Functions} 

\maketitle

\begin{abstract}
Robots built from soft materials will inherently apply lower environmental forces than their rigid counterparts, and therefore may be more suitable in sensitive settings with unintended contact.
However, these robots' applied forces result from both their design and their control system in closed-loop, and therefore, ensuring bounds on these forces requires controller synthesis for safety as well.
This article introduces the first feedback controller for a soft manipulator that formally meets a safety specification with respect to environmental contact.
In our proof-of-concept setting, the robot's environment has known geometry and is deformable with a known elastic modulus.
Our approach maps a bound on applied forces to a safe set of positions of the robot's tip via predicted deformations of the environment.
Then, a quadratic program with Control Barrier Functions in its constraints is used to supervise a nominal feedback signal, verifiably maintaining the robot's tip within this safe set.
Hardware experiments on a multi-segment soft pneumatic robot demonstrate that the proposed framework successfully maintains a positive safety margin.
This framework represents a fundamental shift in perspective on control and safety for soft robots, implementing a formally verifiable logic specification on their pose and contact forces.
\end{abstract}
\vspace{-0.4cm}
\IEEEpeerreviewmaketitle

\section{Introduction}
\label{sec:introduction}
Robots built from soft and deformable materials are often claimed to be inherently `safe' or `safer' than their rigid counterparts \cite{majidi_2014_soft_robots}.
The soft robotics community justifies this claim via the robots' compliance to their environments, deforming upon contact, which applies lower external forces from the robot's body \cite{abidi_intrinsic_2017}.
This understanding aligns well with many intended purposes of soft robots, such as medical operations \cite{enayati2016haptics} or other intimate human-robot contact \cite{vasicSafetyIssuesHumanrobot2013}.

In contrast to this informal claim, the control systems community defines `safety' as satisfaction of a logic specification or constraint \cite{li_formal_2019}.
Given a range of acceptable values of some states of a dynamical system, possibly varying or interleaved over time, maintaining these conditions for all time is \textit{safety}: the set of states or signals is invariant under the dynamics \cite{dang_reachability_1998}.
This invariance perspective (equivalently called reachability \cite{akametalu_reachability-based_2014,kochdumper_provably_2023} 
or persistent feasibility) is a powerful tool that has revolutionized legged robot locomotion \cite{ames_control_2014}, aerial vehicles \cite{funada_visual_2019}, and rigid human-robot interaction \cite{rauscher_constrained_2016,landi_safety_2019}.

\begin{figure}[!t]
    \centering
    \includegraphics[width=1.0\columnwidth]{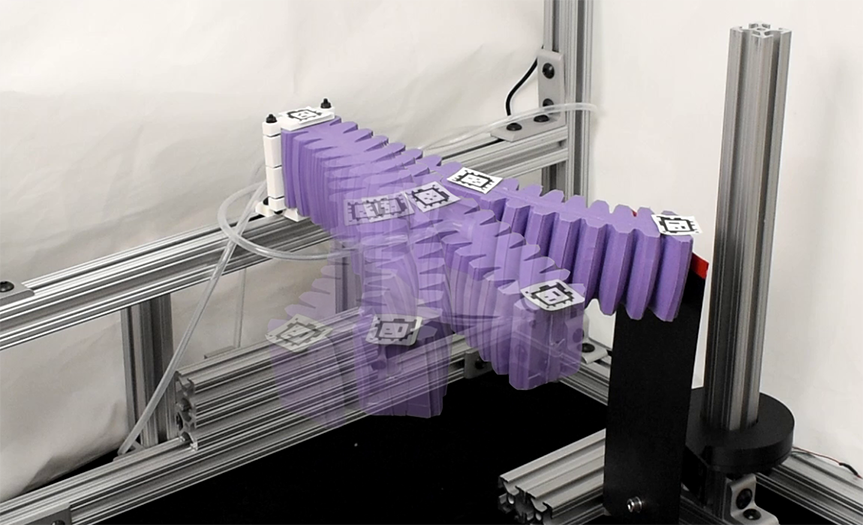}
    \vspace{-0.5cm}
    \caption{This article proposes a feedback control method for a soft robot manipulator (purple) to meet a formal safety specification on its environmental contact forces. We assume the environment (black flexible plate) deforms and is known, and so safe forces map to safe poses. Control barrier functions ensure the robot remains within the set of safe poses.}
    \label{fig:overview}
    \vspace{-0.5cm}
\end{figure}

Prior work attempting to merge these perspectives has shown that \textit{soft robots do not inherently satisfy safety conditions}.
For example, soft actuators can fail by exceeding their operational limits \cite{sabelhaus_safe_2024,anderson_maximizing_2024}, and soft robots can collide with themselves \cite{patterson_safe_2024}. 
Yet no work in soft robot control has proposed a generalizable method for formal safety.
Instead, controllers have focused on stabilization around a desired force or impedance \cite{dyck_impedance_2022,bajo_hybrid_2016}, prevention of mechanical failure \cite{balasubramanian_faulttolerant_2020,anderson_maximizing_2024}, actuator constraints \cite{sabelhaus_safe_2024,feng_safety_2021,garg_autonomous_2025}, or stability and state tracking \cite{li_discrete_2023,della_santina_model-based_2023,bruder_data-driven_2020,haggerty_control_2023,thieffry_control_2019,hachen_non-linear_2025}. 
None of these meet safety constraints on the robot's state or environment: impedance control does not bound forces, static input constraints do not enforce state constraints, and stability is not a replacement for safety \cite{ames_control_2014}.

This manuscript addresses the gap by contributing a first-of-its-kind approach for formal safety in soft robot control, achieving provable set invariance of environmental forces exerted by the robot's end effector (Defn. 1 below).
We adapt the approach of control barrier functions (CBFs) \cite{ames_control_2019}: dynamic constraints on inputs, added to an optimization problem, which ensure invariance of states.
To map states to environmental forces, we consider the proof-of-concept situation where the robot exists in a deformable environment (Fig. \ref{fig:overview}), such as when contacting tissue in a patient, and map deformations to end-effector poses.
We demonstrate that a standard formulation of CBF-based control meets this constraint in both simulation and hardware experiments (Fig. \ref{fig:approach_diagram}), and verify that our controller maintains a positive \textit{safety margin} on force when an open-loop controller does not.

\begin{figure*}[!t]
    \centering
    \vspace{0.2cm}
    \includegraphics[width=1.0\textwidth]{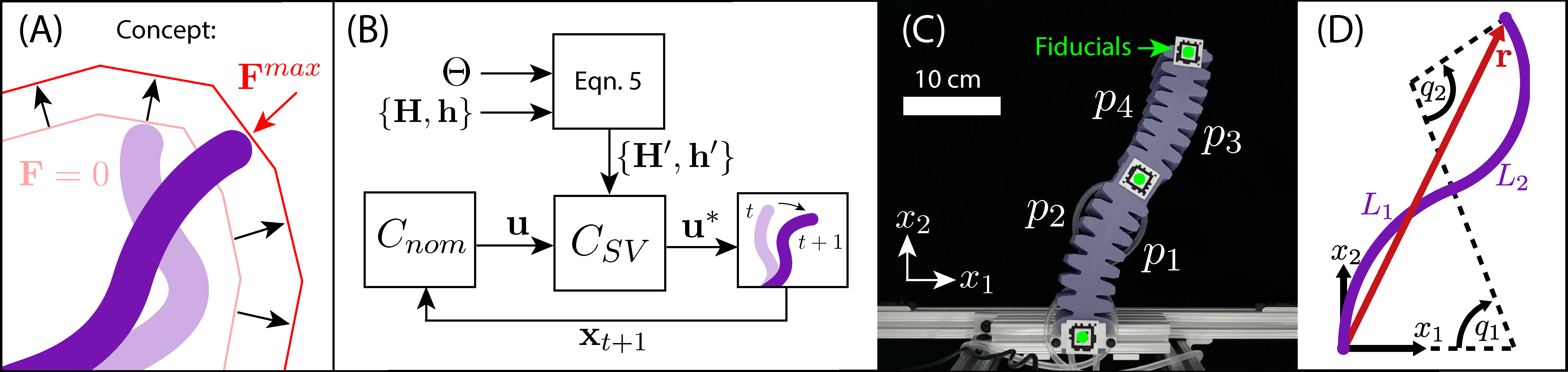}
    \vspace{-0.5cm}
    \caption{Our setup (A) maps a soft robot manipulator's end effector force to its pose by assuming the environment deforms, so a maximum force corresponds to a constraint on the robot's states. Our approach (B) calculates this set as a polytope $\mathbf{H}'\mathbf{r}\leq \mathbf{h}'$ given an undeformed environment surface $\mathbf{H}\mathbf{r} = \mathbf{h}$ and material parameters $\Theta$, where a supervisory controller $C_{SV}$ filters a nominal control signal to maintain the end effector $\mathbf{r}$ in the safe set. Our application (C) is a planar two-segment pneumatic robot with antagonistic actuation chambers $(p_i, p_{i+1})$. Our model (D) is piecewise constant curvature \cite{dellasantina_model_2020}.}
    \label{fig:approach_diagram}
    \vspace{-0.5cm}
\end{figure*}

\vspace{-0.4cm}
\subsection{Paper Contributions}

This article contributes:

\begin{enumerate}
\item A method for mapping environmental forces to a soft manipulator's state space, with a one-to-one set inclusion criterion between them,
\item An adaptation of CBF-based control to a soft robot in this setting, with a proof-by-construction of invariance of the robot's applied forces,
\item A validation in both simulation and hardware, demonstrating a positive safety margin on a set of safe forces.
\end{enumerate} 
\vspace{-0.3cm}
\section{Background and Motivation}
\label{sec:background}

Differences in terminology pervade the discussion of `safety' across the technical disciplines in soft robotics.
The design and modeling communities have considered `safe' motions to be kinematically-feasible workspaces of a soft manipulator \cite{BamdadKinematics2019}.
These `reachable' kinematic workspaces are not formally safe, since they are not \textit{maximum reachable sets} of states under the robot's dynamics \cite{akametalu_reachability-based_2014}.
Other definitions correspond to fault detection \cite{balasubramanian_faulttolerant_2020} or mechanical designs, none of which synthesize control policies for motion.

For control systems, some authors have defined `safety' as a comparison of regions of attraction \cite{feng_safety_2021} or input-to-state stability \cite{stolzle_input--state_2024}, which are not formal conditions as the system may still violate a specification.
Or, considering `safety' as disturbance rejection \cite{patterson_robust_2022} may even be antithetical to softness, as this prevents a soft robot from complying to its environment.
Generally, control of soft robots has yet to settle on a problem statement that de-conflicts the inherent paradox between softness and tracking performance \cite{della_santina_controlling_2017}.

This manuscript proposes a model-based control approach to safety.
Model-free control of soft robots may have good performance on error tracking \cite{bruder_data-driven_2020,haggerty_control_2023,thuruthel_manipulator_2019} but without the provable properties of safety-critical control.
Learned controllers notoriously suffer from unpredictable behaviors in situations outside the training data \cite{pereira_challenges_2020}.

We derive the proposed method for \textit{piecewise constant curvature} (PCC) \cite{dellasantina_model_2020} soft robot dynamics, and apply the method in a proof-of-concept setting: a 2D planar robot, ensuring safety of the end effector, in a known environment.
We adopt this baseline as is common in the development of new controllers for soft robots, as 2D dynamics reduces issues from rotation parameterizations \cite{DellaSantina_OnAnImproved_2020, della_santina_model-based_2023}.
Prior work in force safety in rigid robots has similarly considered the end effector first \cite{rauscher_constrained_2016} before extensions to whole-body safety in unknown environments \cite{landi_safety_2019}.
\vspace{-0.3cm}

\section{Force Safety in Deformable Environments}
\label{sec:safesetcalc}

To the authors' knowledge, this is the first work that allows a soft robot to contact its environment assuming there is some limit of safe force application to that environment.
We take inspiration from prior work in rigid robotics \cite{liu_online_2021} for our definition of force safety, and prior work deformable surfaces \cite{dyck_impedance_2022} for our approach.

\begin{definition}\label{def:forcesafety}
    \underline{\textbf{Force safety.}} Trajectories of a (soft) robot's end effector, $\mathbf{r}(t)$, are force-safe if the applied force at that end effector $\mathbf{F}(t)$ remains within a set of acceptable forces $\mathcal{F}^{safe}$, or equivalently, $\mathcal{F}^{safe}$ is invariant under the system's dynamics: $\mathbf{F}(0) \in \mathcal{F}^{safe} \Rightarrow \mathbf{F}(t) \in \mathcal{F}^{safe} \; \forall t$.
\end{definition}

\vspace{-0.3cm}
\subsection{Problem Setup}

Force safety maps to a set constraint on the robot's kinematics under mild assumptions.
As a proof-of-concept:

\newcounter{assumptionslist}
\begin{enumerate}
    \item The set of manipulator tip positions  $\mathbf{r} \in \mathbb{R}^2$ with no environmental contact is represented by a polytope with $P$-many facets, $\mathcal{N} = \{\mathbf{r} | \mathbf{H} \mathbf{r} \leq \mathbf{h}\}$, with $\mathbf{H}\in \mathbb{R}^{P\times 2}, \mathbf{h} \in \mathbb{R}^P$ having no redundant edges. 
    \item Each facet of the polytope, i.e. $\mathcal{L}_i = \{\mathbf{r} | \mathbf{H}_i \mathbf{r} = h_i\}$, deforms (only) in the normal direction, $\mathbf{\hat n}_i \perp \mathcal{L}_i$.
    \item Surface(s) are deflected to the location of the tip $\mathbf{r}(t)$ in space, $\mathcal{L}_i' = \{\mathbf{r} | \mathbf{H}_i' \mathbf{r} = h_i'\}$ where $\mathbf{H}_i'$ and $h_i'$ are such that $\mathbf{n}_i = n_i \mathbf{\hat n}_i = ||\mathbf{r}(t) - \text{proj}_{\mathcal{L}_i} \mathbf{r}(t)||$ is the normal.
    \item The force-deformation relationship is $\mathbf{F}_i(n_i) = \psi_i(n_i) \mathbf{n}_i$. We only consider normal forces.
    \item The scalar function $\psi_i(\cdot) : \mathbb{R} \mapsto \mathbb{R}$ is strictly monotonic for $n_i > 0$ and is $\psi = 0$ for $n \leq 0$.
    \item The maximum safe normal forces are known, so $\mathcal{F}^{safe} = \{ \mathbf{F} \; | \; \; ||\mathbf{F}_i|| \leq F^{max}_i \; \forall i\}$ with $\mathbf{F} = \sum_i \mathbf{F}_i$.
    \setcounter{assumptionslist}{\value{enumi}}
\end{enumerate}

\noindent Our calculations simplify assuming uniform properties of each facet, chosen in this work for clarity of exposition:

\begin{enumerate}
    \setcounter{enumi}{\value{assumptionslist}}
    \item $\psi_i(\cdot) = \psi(\cdot), \; F^{max}_i = F^{max},  \forall i=1...P$
\end{enumerate}

\noindent Fig. \ref{fig:expansion_diagram} is helpful in interpreting these geometric quantities.

\begin{figure}[!t]
    \centering    
    \includegraphics[width=0.8\columnwidth]{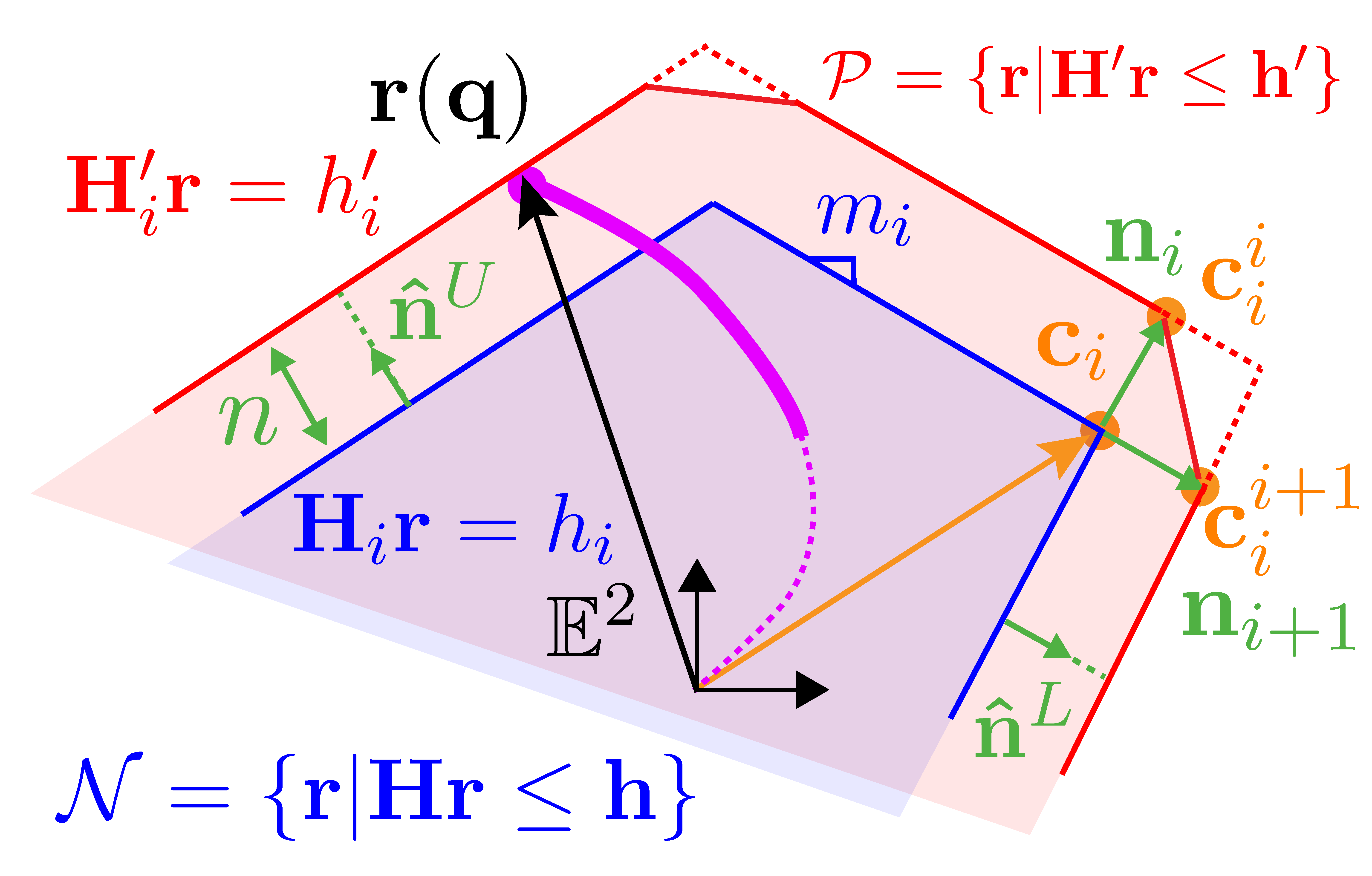}
    \caption{The deformable environment is represented by a no-contact (free space) set, $\mathcal{N}$. We assume each face deforms in its normal direction, $\mathbf{\hat n}^U$ or $\mathbf{
    \hat n}^L$ depending on inequality, upon contact with the end effector. By calculating the maximum deflection $n=n^{max}$ based on a force limit, we convert $\mathcal{N}$ into $\mathcal{P}$, the end effector poses $\mathbf{r}(\mathbf{q})$ that apply less than maximum force. Our approach also adds an additional hyperplane per vertex during the set expansion, between $\mathbf{c}_i^i$ to $\mathbf{c}_i^{i+1}$, to conservatively bound the environment's force when two hyperplanes are deflected.}
    \label{fig:expansion_diagram}
    \vspace{-0.5 cm}
\end{figure}

These assumptions are not inherent limitations of the concept.
Relaxing (1)-(3) corresponds to finer discretizations of $\mathcal{N}$.
A positive Poisson ratio and elastic modulus are sufficient for assumptions (4)-(5).
Assumption (6) has been demonstrated in prior work on tissue mechanics \cite{fujikawa_critical_2017}.
\vspace{-0.3 cm}

\subsection{Construction of a Safe Contact Set}
\label{sec:construction_of_safe_contact_set}
We now calculate a task-space set that maps to $\mathcal{F}^{safe}$:

\vspace{-0.2cm}

\begin{equation}\label{eqn:safeS_forces}
\mathcal{P} = \{\mathbf{r} \;| \; ||\mathbf{F}_i(\mathbf{r})|| \leq F^{max}, \; \; \forall i=1\hdots P\},
\end{equation}

\vspace{-0.2cm}

\noindent i.e., locations in space that the robot's tip can enter while never exerting greater than $F^{max}$, so $\mathcal{P} = \{\mathbf{r} | \mathbf{F}(\mathbf{r}) \in \mathcal{F}^{safe}\}$.

To simplify calculations without loss of generality, we perform a preprocessing step.
Assuming some original $\mathbf{H}_{o}$ and $\mathbf{h}_{o}$ are given for $\mathcal{N}$, 
\vspace{-0.2 cm}
\begin{align}
    & \mathbf{N} = \left(\text{diag}\left(\mathbf{H}_{o} \begin{bmatrix} 0 \\ 1 \end{bmatrix} \right) \right)^{-1}, \quad
    \mathbf{H} = \mathbf{N} \mathbf{H}_{o}, \; \mathbf{h} = \mathbf{N} \mathbf{h}_{o}, \\
    & \therefore \quad \mathbf{H} \mathbf{r} = \mathbf{h} \quad \Rightarrow \quad \begin{bmatrix}
        -m_1 & 1 \\ \vdots \\ m_P & -1
    \end{bmatrix} \mathbf{r} = \begin{bmatrix}
        h_1 \\ \vdots \\ h_P
    \end{bmatrix}
\end{align}

\noindent Readers can interpret this calculation as dividing each row by the element in the second column of $\mathbf{H}_o$, i.e., the facets are now in slope-intercept form.
The lower block of $\mathbf{H}$ corresponds to constraints of the form $\mathbf{H}_i \mathbf{r} \geq h_i$.

Converting the implicit set representation of $\mathcal{P}$ into an explicit representation is trivial given our assumptions.

\begin{lem}\label{lem:safeS_r}
    \underline{\textbf{Force-Safe Kinematics Set.}} The set $\mathcal{P}$ in eqn. (\ref{eqn:safeS_forces}) is equivalent to
    \vspace{-0.3 cm}
    \begin{equation}\label{eqn:safeS_r}
        \mathcal{P} = \{\mathbf{r} | \mathbf{H}' \mathbf{r} \leq \mathbf{h}'\}
    \end{equation}
    \noindent given parameters $\Theta=\{\psi,F^{max}\}$, where each $\mathbf{H}_i'$, $h_i'$ represent the $\mathcal{L}'_i$ with a deformation $n_i = n^{max} =\psi^{-1}(F^{max})$ and:
    \vspace{-0.2 cm}
    \begin{equation}
        \mathbf{H}_i' = \mathbf{H}_i, \quad \quad h_i' = h_i + n^{max} \sqrt{m^2+1}.
    \end{equation} 
\end{lem}

\begin{proof}
    Consider row $i$ so that $\mathbf{H}_i = [-m_i \; \; 1]$ (or respectively, $m_i, \; -1$).
    A short calculation reveals two unit vectors normal to the line $\mathbf{H}_i \mathbf{h} = h_i$,
    \begin{equation}
        \mathbf{\hat n}_i^U = \frac{1}{\sqrt{m_i^2+1}}\begin{bmatrix}
            -m_i \\ 1
        \end{bmatrix}, \quad \quad \mathbf{\hat n}_i^L = -\mathbf{\hat n}^U_i.
    \end{equation} 
    \noindent Here, $\mathbf{\hat n}^U$ is in positive $\mathbf{E}_2$, as $\mathbf{\hat n}^U\cdot\mathbf{E}_2 > 0$, and vice-versa.
    
    Then consider $\mathbf{n}_i=n_i\mathbf{\hat n}_i$. 
    Since $\psi>0$ is invertible per (4)-(5), $\psi(n_i) \leq F^{max} \Rightarrow n_i \leq \psi^{-1}(F^{max})$ per assumptions (5)-(6). Setting $n^{max} = \psi^{-1}(F^{max})$, our safety inequality becomes $n_i \leq n^{max} \; \; \forall i$, i.e., $\mathbf{n}_i^{max}=n^{max}\mathbf{\hat n}_i$.
    A point on the facet $\mathbf{H}_i \mathbf{r} = h_i$ translated by $\mathbf{n}_i^{max}$, choosing $(\cdot)^U$ or $(\cdot)^L$ accordingly, is $\mathbf{r} = s \; [1 \; \; m_i]^\top + [0 \; \; h_i]^\top + n^{max} \mathbf{\hat n}_i$, where $s$ is an arbitrary parameter.
    With some arithmetic,
    \begin{equation}\label{eqn:hi_prime}
    \begin{bmatrix}
        -m_i & 1
    \end{bmatrix} \mathbf{r} = h_i + n^{max} \sqrt{m_i^2+1}    
    \end{equation}
    \noindent for both $\mathbf{\hat n}^U$ and $\mathbf{\hat n}^L$.
    Observe $h_i' = h_i + n^{max} \sqrt{m_i^2+1}$ and $\mathbf{H}_i' = \mathbf{H}_i$, so by construction:
    \[
    \mathbf{H}'_i \mathbf{r} \leq h'_i \Rightarrow n(\mathbf{r})_i \leq n^{max} \Rightarrow||\mathbf{F}_i|| \leq F^{max}.
    \]
    \noindent Apply eqn. (\ref{eqn:hi_prime}) for all rows $i$ to obtain the intersection $\mathbf{H}' \mathbf{r} \leq \mathbf{h}'$ where $||\mathbf{F}_i|| \leq F^{max} \; \forall i$, which is $\mathcal{P}$ in eqn. (\ref{eqn:safeS_r}). 
\end{proof}

\begin{remark}
    Lemma \ref{lem:safeS_r} is a mathematically formal way of stating our intuition: ``translate each line in $\mathcal{N}$ outwards by a distance $n^{max}$ to obtain $\mathcal{P}$.''
\end{remark}
\vspace{-0.3cm}
\subsection{A Correction at the Vertices}

In some instances, the set $\mathcal{P}$ does contain regions where $||\mathbf{r} - \text{proj}_\mathcal{N} \mathbf{r}|| > n^{max}$: the tip is too far away from the no-contact set.
In particular, when $\text{proj}_\mathcal{N} \mathbf{r}$ is one of the vertices in $\mathcal{N}$, multiple constraints are active ($\mathbf{F} = \mathbf{F}_i + \mathbf{F}_{i+1}$) and $n_i \leq n^{max} \; \forall i$ is insufficient.
We propose a conservative under-approximation of the safe region that maintains convexity and linearity of the calculation.
Define the two vectors $\mathbf{c}_i^i = \mathbf{c}_i + \mathbf{n}_i$ and $\mathbf{c}_i^{i+1} = \mathbf{c}_i + \mathbf{n}_{i+1}$ per Fig. \ref{fig:expansion_diagram}.
Then, concatenate row $P+i+1$ to $\{\mathbf{H}',\mathbf{h}'\}$ with $m$ and $h$ as the slope/intercept of the line passing through these points, with a sign correction as needed ($\mathbf{\hat n}^L=-\mathbf{\hat n}^U$).
Then $n_{P+i+1} \leq n^{max}$, satisfying Lemma \ref{lem:safeS_r} by eliminating the dashed regions in Fig. \ref{fig:expansion_diagram} from $\mathcal{P}$.

\vspace{-0.3cm}
\section{Kinematics and Dynamics Model}
\label{sec:KinematicsandDynamics}

For our proof-of-concept, we use the piecewise constant curvature (PCC) kinematics of a soft robot manipulator with the augmented body model of Della Santina et al. \cite{dellasantina_model_2020} to convert our task-space force safety into state space, giving both $\mathbf{r}(\mathbf{q})$ with states $\mathbf{q} \in \mathbb{R}^N$ as well as $\dot{\mathbf{x}} = f(\mathbf{x}, \mathbf{u})$.
We briefly summarize, deferring to prior work for details.

Here, the robot is represented with one subtended angle $q_i$ per CC segment of length $L_i$ for $i=1\hdots N$, see Fig. \ref{fig:approach_diagram}.
PCC kinematics admit a transformation matrix $\mathbf{T}_i^{i+1}(\mathbf{q})$ from one CC segment to the next, giving the end effector position as $\mathbf{r}(\mathbf{q}) = (\sum_{i=0}^{N-1}\mathbf{T}_i^{i+1}(\mathbf{q})) \mathbf{r}_0$.

For the dynamics, the augmented body model represents one CC segment as a series of $z$-many translation and rotation joints, states $\xi_i \in \mathbb{R}^z$. 
A user then selects a kinematic constraint $m_i(\cdot):\mathbb{R} \mapsto \mathbb{R}^{z}$, where $m_i(q_i) = \xi_i$, aligning e.g. the center of mass of the PCC segment.
For this manuscript, we choose Della Santina et. al.'s RPPR augmented body, $m_i = [\frac{q_i}{2}, \; L_i \frac{sin(q_i)/2}{q_i},\; L_i \frac{sin(q_i)/2}{q_i}, \frac{q_i}{2}]^\top$, with one mass between the prismatic joints.
Some applications of the chain rule on the full $N$-segment state, $\xi = m(\mathbf{q})$, give $\dot{\xi} = \mathbf{J}_m(\mathbf{q})\dot{\mathbf{q}}$ and $\ddot{\xi} = \dot{ \mathbf{J}_m}(\mathbf{q},\dot{\mathbf{q}})\dot{\mathbf{q}} + \mathbf{J}_m(\mathbf{q})\ddot{\mathbf{q}}$.
Inserting these Jacobians into the manipulator equation for $\ddot \xi$, assuming some elastic damping and spring forces at each joint, converts a $Nz$-dimensional ODE into our $N$-dimensional expression:
\begin{equation}\label{eqn:softdynamics_noinput}
    \mathbf{M}(\mathbf{q})\ddot{\mathbf{q}} + \mathbf{C}(\mathbf{q}, \dot{\mathbf{q}})\dot{\mathbf{q}}+ \mathbf{D}\dot{\mathbf{q}} + \mathbf{K}\mathbf{q} = \mathbf{J}^\top f_{ext}
\end{equation}
For the control input, past work \cite{dellasantina_model_2020} demonstrates an approximation of pneumatic actuation as external forces, $\mathbf{J}^\top f_{ext} = \mathbf{\tau}$, viewed as virtual torques on the body.
And as in past work \cite{della_santina_controlling_2017,dickson_real-time_2025}, we calibrate these virtual torques using a least squares fit to pneumatic pressures $\mathbf{u} \in \mathbb{R}^N$.
So, $\mathbf{\tau} = \Lambda\mathbf{u}$ with $\Lambda = \text{diag}([\lambda_1, \; \hdots, \; \lambda_N])$ fit from data.

Minor rearranging gives the final equations of motion with states $\mathbf{x} = [\mathbf{q}^\top \; \dot{\mathbf{q}}^\top]^\top$ as $\dot{\mathbf{x}} = f(\mathbf{x}) + g(\mathbf{x})\mathbf{u}$, with:
\begin{equation}\label{eqn:dynamics}
 f = \begin{bmatrix}
\dot{\mathbf{q}} \\
\mathbf{M}^{-1} \left(\mathbf{C} \dot{\mathbf{q}} + \mathbf{D} \dot{\mathbf{q}} + \mathbf{K} \mathbf{q} \right)
\end{bmatrix}, \quad
g = \begin{bmatrix}
0 \\
\mathbf{M}^{-1} \Lambda
\end{bmatrix}   .
\end{equation}
\noindent These dynamics are control-affine, key to the method below.

\vspace{-0.2cm}
\section{Control of Soft Robot Manipulator Forces Using Control Barrier Functions} \label{sec:Safety-Critical}

We now synthesize a control system to render $\mathcal{P}$ invariant, given the control-affine $f(\mathbf{x})$ and $g(\mathbf{x})$.

\vspace{-0.3cm}

\subsection{Control Barrier Functions For Force Safety}
\label{sec:CBFforforcesafety}

Many techniques can enforce invariance of a set in a control-affine dynamical system, but the approach of \textit{control barrier functions} (CBFs) has the benefit of real-time feedback, possible by formulating a quadratic program \cite{ames_control_2014}.
To apply CBFs to our problem, we follow Rauscher et al. \cite{rauscher_constrained_2016} in defining a set of scalar constraints $b_i(\cdot) : \mathbb{R}^{2N} \mapsto \mathbb{R}$, such that the set of safe states $\mathcal{C}$ is defined as superlevel sets of these constraints, $\mathcal{C} = \{\mathbf{x} \; | \; b_i(\mathbf{x}) > 0 \; \forall i\}$.
For our problem statement, these are the rows of the polytope $\mathcal{P}$ with the forward kinematics inserted,
\begin{equation}\label{eqn:b_i}
b_i(\mathbf{x}) = h_i' - \mathbf{H}_i'\mathbf{r}(\mathbf{q}), \quad \quad i=1\hdots P.
\end{equation}
Next, we must define a control barrier function $B_i(\mathbf{x})$ per constraint, each of which must be the inverse of a class-$\kappa$ function and be such that $\mathcal{L}_g B_i(\mathbf{x}) \neq 0$, where $\mathcal{L}$ is the Lie derivative.
We refer to prior work for more detailed justifications \cite{rauscher_constrained_2016} (Defn. 3-4).
Considering our soft robot manipulator with dynamics of the form in eqn. (\ref{eqn:softdynamics_noinput})-(\ref{eqn:dynamics}), our constraint function $b_i(\mathbf{r}(\mathbf{q}))$ has a relative degree two since $\ddot{b}_i$ depends on the system's control input $\mathbf{u}$. 
Consequently, we adopt the barrier function from \cite{rauscher_constrained_2016,Hsu2015Control},
\begin{equation}
\label{eq:barrierfunction}
B_i(\mathbf{x}) = - \ln \left( \frac{b_i(\mathbf{x})}{1 + b_i(\mathbf{x})} \right) + a_E\frac{b_E \, \dot{b_i}(\mathbf{x})^2}{1 + b_E \, \dot{b_i}(\mathbf{x})^2},
\end{equation}
\noindent where $b_i(\mathbf{x})$ and $\dot{b}_i(\mathbf{x})$ are the constraint function and its derivative respectively while $a_E$ and $b_E$ are tuning constants.
Prior work has shown that eqn. (\ref{eq:barrierfunction}) is a valid CBF for relative degree two systems.
We note that small values of $a_E$, $b_E > 0$ result in more conservative behavior of the soft robot manipulator's end effector, while larger values allow for less restrictive motion near the safety constraint. 

Using this barrier function, we can write a constraint that, if satisfied by a control input $\mathbf{u}$, ensures that the system states $\mathbf{x}$ remain within the $\mathcal{C}$ set (per \cite{patterson_safe_2024,rauscher_constrained_2016,ames_control_2014,ames_control_2019}),
\begin{equation}\label{eqn:CBFconstraint}
L_f B_i(\mathbf{x}) + L_g B_i(\mathbf{x}) \, \mathbf{u} - \frac{\gamma_i}{B_i(\mathbf{x})} \leq 0 \quad \forall i,
\end{equation}
\noindent where $\gamma_i > 0$ is an additional tuning constant that indicates the rate of growth of the barrier function (higher $\gamma$ is less conservative).
Using the $f$ and $g$ from the dynamics in eqn. (\ref{eqn:dynamics}), these Lie derivatives become:
\begin{equation}
\mathcal{L}_f B = \mathbf{J}_B f = \mathbf{J}_B \dot{\mathbf{x}}, \quad \mathcal{L}_g B = \mathbf{J}_B g = \frac{\partial B}{\partial \dot{\mathbf{q}}}\mathbf{M}^{-1}\Lambda
\end{equation}
\noindent where $\mathbf{J}_B$ is the Jacobian of the barrier function \eqref{eq:barrierfunction}, taking advantage of the robot's kinematics as $\mathbf{J}_B = [\partial B/\partial \mathbf{q}^\top, \; \; \partial B/\partial\dot{\mathbf{q}}^\top]^\top$.

\vspace{-0.3cm}
\subsection{Feedback Controller with Control Barrier Functions}

Our goal then becomes calculation of a feedback control law $\mathbf{u}(\mathbf{x})$ that satisfies eqn. (\ref{eqn:CBFconstraint}) while also attempting to perform some nominal task.
Here, we assume that there is an existing (unsafe) nominal controller, $\mathbf{u}_{nom}(\mathbf{x})$, for example the computed torque law from \cite{dellasantina_model_2020}.
Following the literature \cite{rauscher_constrained_2016}, we adopt a quadratic program (QP) that minimizes the squared norm of the error between this nominal signal and the applied signal, $||\mathbf{u} - \mathbf{u}_{nom}||^2 = \mathbf{u}^\top \mathbf{u} - 2 \mathbf{u}^\top \mathbf{u}_{nom} + \mathbf{u}_{nom}^\top \mathbf{u}_{nom}$.
Dropping the constant term and inserting eqn. (\ref{eqn:CBFconstraint}) as a constraint, our calculation becomes:
\begin{equation}
\begin{aligned}
\label{eq:quadprog}
\mathbf{u}^*(\mathbf{x}) = \; &\arg\min_{\mathbf{u}} && \mathbf{u}^\top \mathbf{u} - 2 \mathbf{u}^\top \mathbf{u}_{nom} \\
&\text{s. t.} && \mathbf{A}(\mathbf{x}) \mathbf{u} \leq \mathbf{b}(\mathbf{x})
\end{aligned}
\end{equation}
\noindent where $\mathbf{u}$ is our decision variable since $\mathbf{x}$ is held constant at each timestep, and
\begin{equation}
\begin{aligned}
\mathbf{A}(\mathbf{x}) &= 
\begin{bmatrix}
\mathcal{L}_{g_1} B_1 \; \cdots \; \mathcal{L}_{g_N} B_1 \\
\vdots \quad\quad\quad \vdots \\
\mathcal{L}_{g_1} B_P \; \cdots \; \mathcal{L}_{g_N} B_P
\end{bmatrix}, &
\mathbf{b}(\mathbf{x}) = 
\begin{bmatrix}
\frac{\gamma_1}{B_1} - \mathcal{L}_f B_1 \\
\vdots \\
\frac{\gamma_P}{B_P} - \mathcal{L}_f B_P
\end{bmatrix}
\end{aligned}
\tag{15}
\end{equation}
\noindent This controller satisfies the following safety condition:
\begin{theorem}\label{thm:invarianceP}
    \ul{\textbf{Force Safety Under CBF-Based QP Feedback.}} If a soft robot governed by PCC kinematics and the dynamics of eqn. (\ref{eqn:softdynamics_noinput})-(\ref{eqn:dynamics}) starts at $\mathbf{q}(0) \in \mathcal{C}$ defined by eqn. (\ref{eqn:b_i}), applying the control law $\mathbf{u}(\mathbf{x}(t)) = \mathbf{u}^*(\mathbf{x}(t))$ by solving eqn. (\ref{eq:quadprog}) renders the set $\mathcal{P}$ from eqn. (\ref{eqn:safeS_r}) invariant, i.e., $\mathbf{F}(t) \in \mathcal{F}^{safe} \; \; \forall t>0$.
\end{theorem}

\begin{proof}
    From Theorem 1 of Rauscher et al. \cite{rauscher_constrained_2016}, the control law of eqn. (\ref{eq:quadprog}) and barrier function (\ref{eq:barrierfunction}) will ensure that $b_i(\mathbf{x}(t)) > 0  \; \forall t$.
    We have defined $\mathcal{C}$ as this set, now invariant.
    Correspondingly, $\mathbf{x} \in \mathcal{C} \Rightarrow \mathbf{r}\in \mathcal{P}$.
    By Lemma \ref{lem:safeS_r}, $\mathbf{r}(t) \in \mathcal{P} \; \forall t$ meets Definition \ref{def:forcesafety}.
\end{proof}

\begin{remark}
    Readers are referred to prior work for details such as Lipschitz continuitiy conditions, admissibility, set intersections of multiple constraints, etc., as Theorem \ref{thm:invarianceP} does not introduce any novel theoretical considerations.
\end{remark}

\section{Simulation Validation}

Our approach was validated in a dynamics simulation of eqn. (\ref{eqn:softdynamics_noinput})-(\ref{eqn:dynamics}).
Results show that Theorem \ref{thm:invarianceP} holds.

\subsection{Simulation Setup and Calibration}

This manuscript considers soft robots which are control-affine in their dynamics, for which pneumatics are the common realization \cite{dellasantina_model_2020}.
Consequently, the simulation from \cite{dickson_real-time_2025} was used for this validation task, now calibrated against a physical prototype of a two-segment planar soft robot actuated by antagonistic pneumatic pressure inputs (Fig. \ref{fig:approach_diagram}).
Each segment has length 0.122 m and mass 0.13 kg.
Here, the robot is fully actuated, by mapping negative control inputs to opposing pneumatic chambers $p_j,p_{j+1}$.
We calibrate the constants in the $\mathbf{K}, \mathbf{D}$ and $\Lambda$ matrices of the dynamics model using a least squares fit of data points obtained from the hardware prototype's open-loop motion using motor babbling, as suggested by \cite{dellasantina_model_2020}. 
The simulation uses forward Euler numerical integration with $\Delta_t=1e^{-5}$ sec, with a control frequency of 1000 Hz. 

\vspace{-0.4cm}
\subsection{Nominal Controller}

As proof-of-concept, we employ an open-loop sequence of sinusoids for $\mathbf{u}_{nom}(t)$, mimicking worst-case random motions in a sensitive environment: $u_{nom,j} = A_j\sin(2\pi f t)$.
We tune the amplitudes and frequencies to achieve desired coverage of the task space, approximately $\pm 25^\circ$ bending angles, with $A_j=80$ hPa (hecto-Pascal) and $f=0.01$ Hz.

\vspace{-0.25cm}
\subsection{Safe Controller Tests}

We implemented eqn. (\ref{eq:quadprog}) in MATLAB's {\texttt{quadprog}}.
The environment was taken as one constraint for clarity of exposition, with $\mathbf{H}_1 = [0.36, \; 1.0]$, $h_1 = 0.15$, and $F^{max} = 0.16$.
The deformation was modeled as a linear spring, $\psi(n) = k n$, with a spring constant of $k = 11.16$.
These values correspond to visual identification of the robot's motion in hardware, contact with an environment approximately halfway through its sinusoid, with an anticipated maximum displacement of the environment of $n^{max} \approx 1.5$ cm.

We executed the controller with three sets of tuning constants for the $B_1(\mathbf{x})$ CBF, shown in Table \ref{tab:simresults}, corresponding to conservativeness (None, Low, High), chosen by hand.
`None' removes the constraint ($\mathbf{u}^* = \mathbf{u}_{nom}$).
For each simulation, we postprocess to calculate the end effector position $\mathbf{r} = [r_x, \; r_y]^\top$, as well as the simulated force and simulated \textit{safety margin}.
We define the safety margin of the closed-loop trajectory in simulation as a normalized distance to the boundary of our safe set  $\mathcal{P}$, i.e., $\rho(t) = (F^{max} - F(t))/F^{max} = (n^{max}-n(t))/n^{max}$, since the simulation has a one-to-one mapping between force and displacement. 
A safety margin of $\rho = 1$ indicates that the robot is not contacting the environment, and $\rho < 0$ indicates constraint violation.  
To calculate $F(t)$ and $\rho(t)$, we calculate $n(\mathbf{r}(t))$ then invert the force-displacement relationship.

\begin{figure}[!t]
    \centering
\includegraphics[width=1.0\columnwidth]{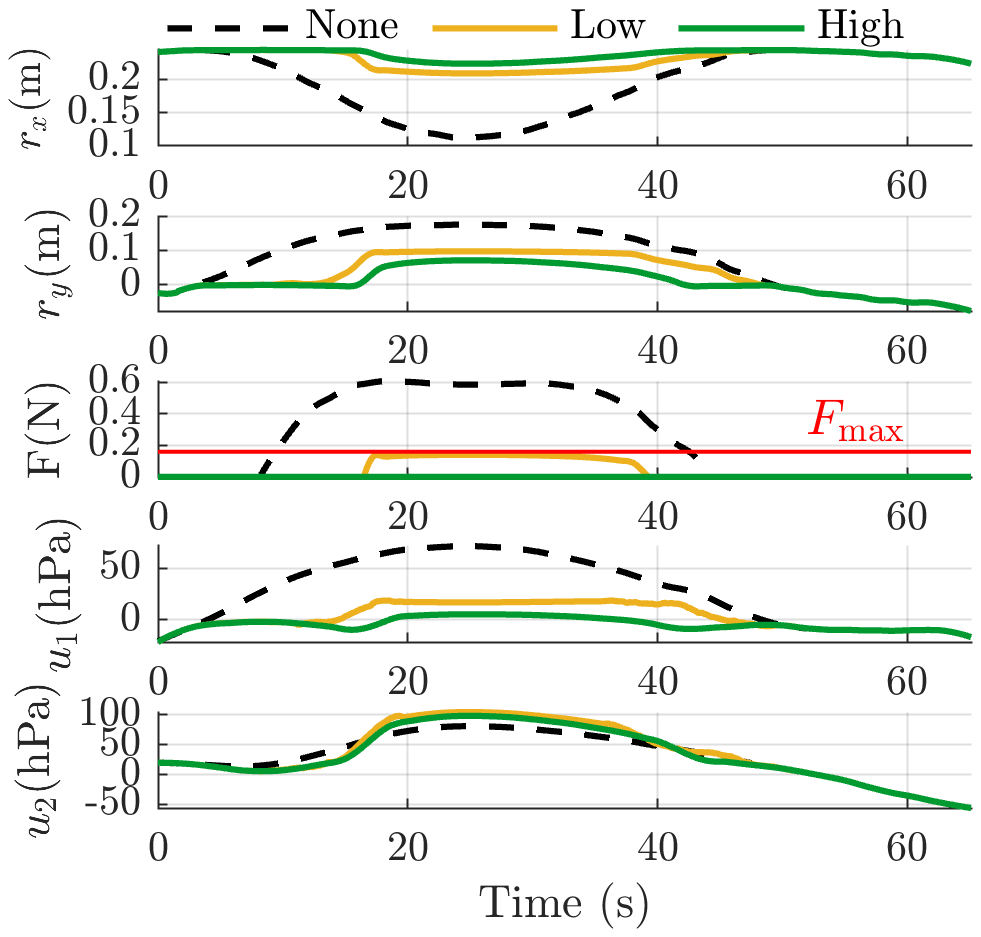}
    \vspace{-0.7cm}
    \caption{Simulation results with three levels of conservativeness of CBF tuning constants all show that the robot's end effector is prevented from moving out of the unsafe set, when the unsafe $\mathbf{u}_{nom}$ control (black dashed line) would violate safety.}
    \vspace{-0.2cm}
    \label{fig:sim_cbfplot}    
\end{figure}

\begin{figure}[!t]
    \centering
\includegraphics[width=0.9\columnwidth]{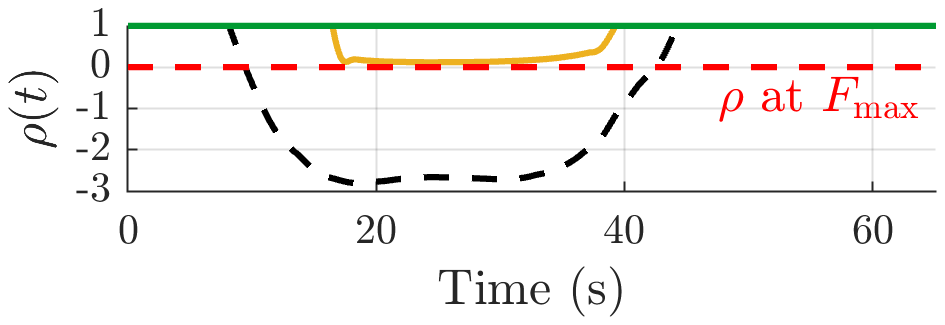}
    \vspace{-0.2cm}
    \caption{All simulations with CBF-based control show a positive safety margin on force application. The most conservative CBF tuning (green) prevents all environmental contact.}
    \label{fig:rho_sim}
    \vspace{-0.4cm}
\end{figure}

\vspace{-0.2cm}
\subsection{Simulation Results}\label{sec:sim}

\begin{table}[!b]
\centering
\vspace{-0.4cm}
\caption{Simulation Experiment Safety Margin vs. Tuning of CBF}
\label{tab:simresults}
\begin{tabular}{|c|c|c|c|c|}
\hline
Conservativeness & $\mathbf{a_E}$ & $\mathbf{b_E}$ & $\mathbf{\gamma}$ & $\mathbf{\rho_{min}}$\\
\hline
High & $0.2$ & $0.2$  & $0.2$ & $1.0$\\
\hline
Low & $2.0$ & $2.0$  & $2.0$  & $0.1208$ \\
\hline
None & $0.0$ & $0.0$  & $0.0$ & $-2.808$ \\ 
\hline
\end{tabular}
\end{table}

Theorem \ref{thm:invarianceP} holds in all our nonzero-CBF simulations (Fig. \ref{fig:sim_cbfplot}), as we observe $\rho(t) > 0$ in Fig. \ref{fig:rho_sim}.
Without our CBF controller in the loop (`None,' black line), the robot violates the safe force constraint, third panel of Fig. \ref{fig:sim_cbfplot}, and reaches a negative safety margin $\rho_{min}$ of $-2.808$.
As we introduce the CBF controller at a low level, the safety margin increases and stays above the critical threshold $\rho = 0$, see Table \ref{tab:simresults}. 
At a high level of conservativeness, the safety margin remains constant at $\rho = 1$, indicating that the robot avoids contact with the environment entirely. Using MATLAB’s {\texttt {tic toc}}, the average solve time for $\mathbf{u}^*(\mathbf{x})$ in eqn. (\ref{eq:quadprog}) was $0.00031$s.

\section{Hardware Experiment Setup}
\label{sec:hw}
We then tested the CBF-based controller in a hardware experiment corresponding to the simulation.
As with the simulation validation, Theorem \ref{thm:invarianceP} held in all cases.

\vspace{-0.3 cm}
\subsection{Hardware Platform}

The soft limb was then used for a real-time control test (Fig. \ref{fig:hw_overview}).
The limb was placed in the environment with fiducial markers (AprilTags \cite{wang_apriltag_2016}) located at the start of the limb and at the end of every segment, captured with an overhead camera, with the bending angles $\mathbf{q}$ calculated from the tag locations. 
The average frame rate of the AprilTag detector was 12.8 fps.
The electronics and software infrastructure were based on prior work \cite{anderson_maximizing_2024}, with feedback at at 10 Hz.
The chambers' pressure was controlled by a microcontroller with a low-level PID loop over a proportional valve  (iQ Volta), recorded by a pressure sensor (Honeywell MPRLS). 
Lastly, a flexible sheet of ABS was attached to a force gauge for use as the deformable surface.
Its spring constant was calibrated using masses placed at the tip under gravitational loading, then was inserted into the test fixture aligned with the robot's tip.

To calibrate the location of the force plate that matched the simulation setup, it was first approximately aligned in the plane by measurements from the origin.
Then, a constant input was applied to the robot, $\mathbf{u} = \mathbf{u}_{T}$, where $\mathbf{u}_{T}$ was taken from the simulation at the first point when the robot's end effector exits the no-contact set $\mathcal{N}$.
The force plate was then adjusted manually while the robot was pressurized until the force readings became (almost) zero.

\begin{figure}[!t]
    \centering
    \vspace{0.2cm}
    \includegraphics[width=0.8\columnwidth]{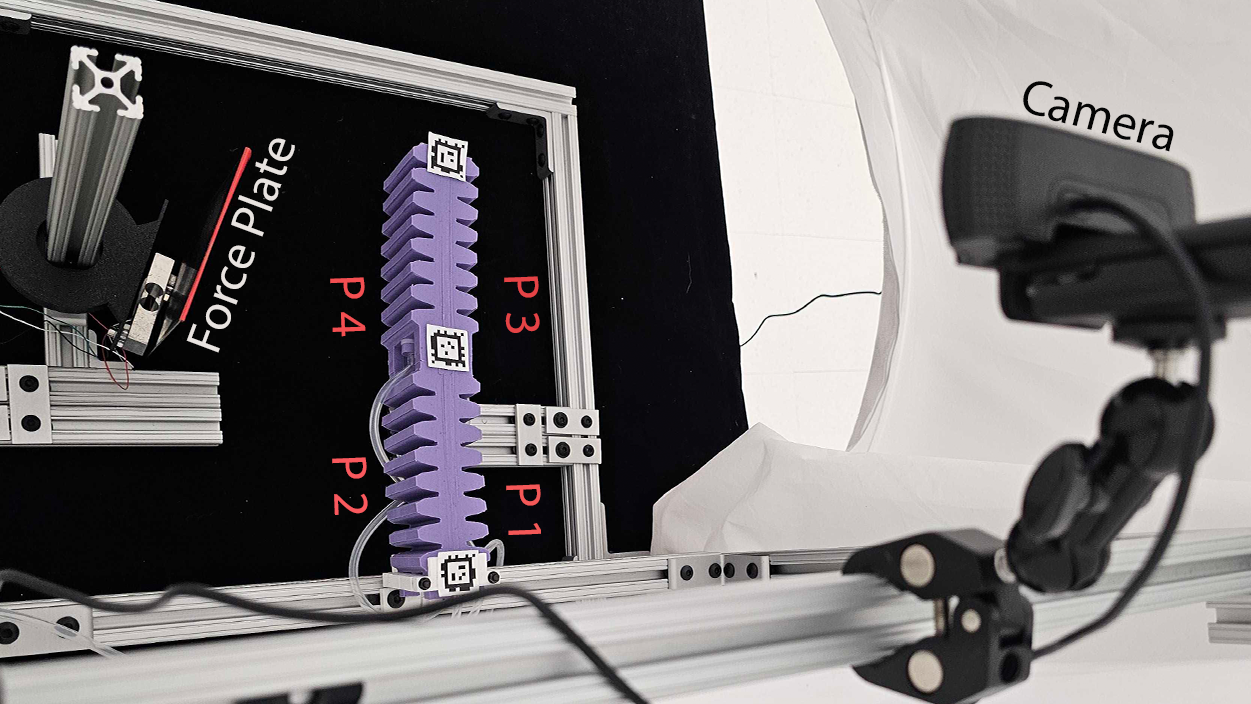}
    \vspace{-0.2cm}
    \caption{Test setup with soft robot limb for hardware validation of CBF-based control framework. A camera placed above the limb captures states $\mathbf{q}$ via three Apriltags placed per-segment. Control inputs from the CBF are fed into the low level controller to increase/decrease the measured pressures ($p_1$, $p_2$, $p_3$, $p_4$) in the limb's chambers. The force plate was calibrated to match the position in simulation and force measurements are used to demonstrate that force safety was achieved.}
    \label{fig:hw_overview}
    \vspace{-0.5cm}
\end{figure}

\vspace{-0.1cm}
\section{Hardware Experimental Results}
\label{sec:hw_results}

We then repeated the same tests as in simulation.
This setup used a ROS2 MATLAB node to receive bending angles from OpenCV and send the outputs of {\texttt{quadprog}} to the microcontroller.
Five trials per conservativeness level were performed, with means and standard deviations ($\pm 2 \sigma$) plotted in Fig. 7 and 8. 

Results were remarkably consistent across trials with low variance, and all trials with an active CBF maintained a positive safety margin (Table \ref{tab:experiment}, Fig. \ref{fig:HWResultsoverview}, \ref{fig:HWRho}).
As expected, decreasing conservativeness leads to a smaller worst-case safety margin $\rho_{min}$, as the robot's tip is allowed closer to the unsafe set boundary.
The `High' conservative tuning makes contact with the flexible force plate, but maintains a force below the limit.
The `Low' conservative tuning maintains a positive safety margin but makes contact with a slightly larger force.
Without the CBF (`None'), the force constraint is violated.
The control input traces (Fig. \ref{fig:HWResultsoverview}) show that the CBF controller reshapes the nominal pressure waveforms primarily during the critical contact period (approximately 15 to 55 sec), preserving the shape of the input elsewhere, as seen in simulation. 
The three test cases in our simulation and hardware tests (different tunings of the CBF) represent a qualitative evaluation of the robustness of our approach in our known environment: the tuning does not affect the validity of our safety proof.

\begin{figure}[!h]
    \centering
    \includegraphics[width=1.0\columnwidth]{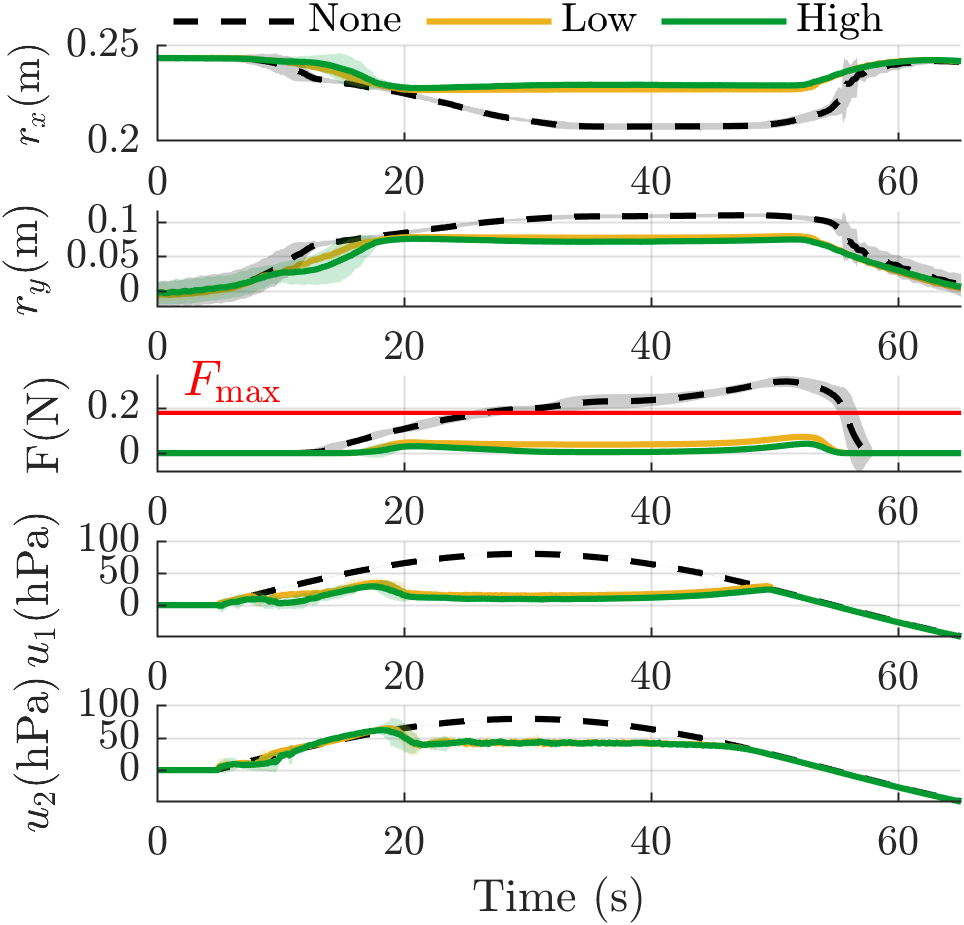}
    \vspace{-0.6cm}
    \caption{Hardware experiment end-effector position $r_x$ and $r_y$ in meters, force $F$ in Newtons, and controller inputs $u_1$ and $u_2$ in hPa for three levels of conservativeness in the CBFs: None, Low, High. Including the CBF ensures the limb never exceeds $F_{max}$.}
    \vspace{-0.2cm}
    \label{fig:HWResultsoverview}
\end{figure}

\begin{figure}[!h]
    \centering
    \includegraphics[width=1.0\columnwidth]{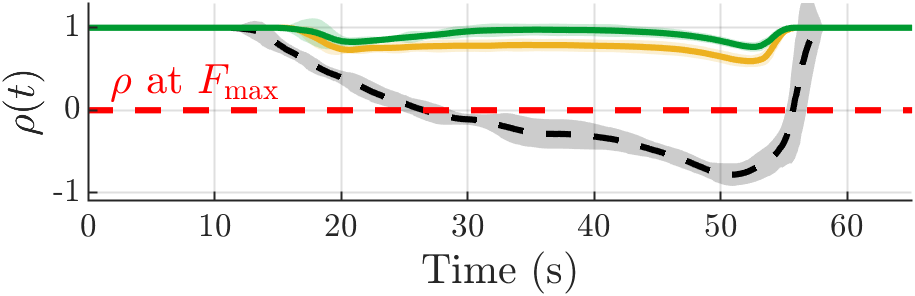}
    \vspace{-0.6cm}
    \caption{Safety margin $\rho(t)$ for the hardware experiments. The red dashed line is the minimum  $\rho$ value for force safety, satisfied for all nonzero CBFs.}
    \label{fig:HWRho}
     \vspace{-0.1cm}
\end{figure}

\subsection{Visualizations}
\label{sec:hw_viz}

To provide a clear visual comparison of the simulation results, Fig. \ref{fig:CBF_timelapse} overlays the simulation's output at three key points versus the hardware test.
The kinematics align well qualitatively, indicating a good calibration of our PCC model.
The magnitude of the deformation of the flexible force plate is also visible.
\begin{table}[b]
\centering
\caption{Hardware Experiment Safety Margin vs. Tuning of CBF}
\label{tab:experiment}
\begin{tabular}{|c|c|c|c|c|}
\hline
Conservativeness & $\mathbf{a_E}$ & $\mathbf{b_E}$ & $\mathbf{\gamma}$ & $\mathbf{\rho_{min}}$\\
\hline
High & $10$ & $10$  & $10$ & $0.767$\\
\hline
Low & $100$ & $100$  & $100$  & $0.594$ \\
\hline
None & $0.0$ & $0.0$  & $0.0$ & $-0.780$ \\ 
\hline
\end{tabular}
\end{table}

\begin{figure}[th]
    \centering
    \includegraphics[width=1\columnwidth]{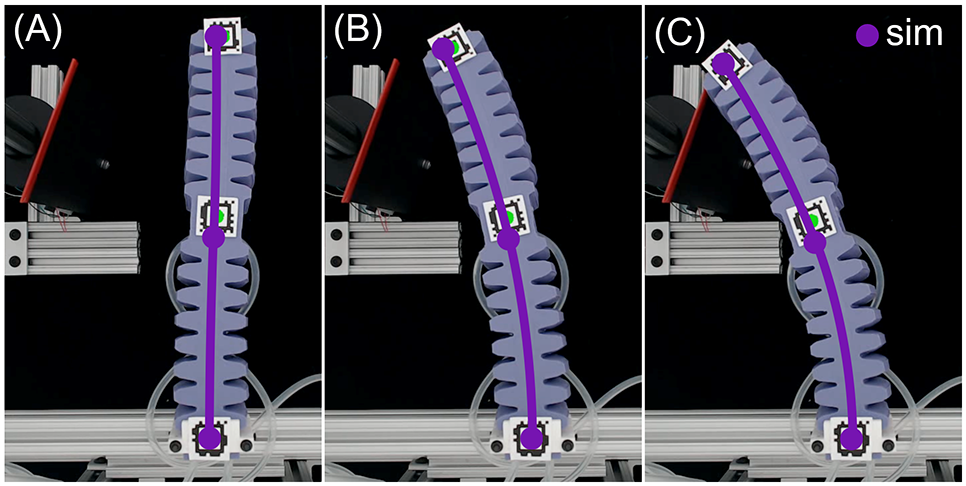}
    \vspace{-0.65cm}
    \caption{Comparison of the simulation's kinematics versus hardware at three key points during a corresponding test demonstrates alignment.} 
    \label{fig:CBF_timelapse}
    \vspace{-0.2cm}
\end{figure}

\vspace{-0.5cm}
\section{Discussion}
\label{sec:dis}

Results in both simulation and hardware experiments confirm that our proposed safety-critical control approach successfully satisfies an invariance condition on the force applied by a pneumatic soft robot's tip. 
At all levels of tuning the controller's constants for conservativeness, the trends in minimum safety margin $\rho_{\min}$, tip position $r_x$, $r_y$, force profiles $F(t)$, force plate deflection $n$ and control inputs $u_1$, $u_2$ remain consistent across both simulation and hardware.

\vspace{-0.3cm}
\subsection{Sources of Error}

Although our results are consistent in terms of safety margin, multiple sources of error exist that cause the trajectories between simulation vs. hardware to be qualitatively different (compare Fig. \ref{fig:sim_cbfplot} vs. Fig. \ref{fig:HWResultsoverview}).
Our setup's assumptions include the PCC dynamics and no gravity, but also a high bandwidth pressure input (whereas our hardware has a finite response time), fast sensing of positions and velocities (whereas our computer vision has a slow frame rate and a finite difference for $\dot{\mathbf{q}}$), normal displacement of the environment (whereas the hardware has some small amount of 3D motion), environmental properties (linear stiffness) and all other typical hardware difficulties (camera distortion, unmodeled disturbances, etc).
Future work using higher-frequency motion capture could analyze practical limitations of our safety proof.

The hardware unsafe motion shows a blocking behavior when the robot is pushed up against the flexible force plate, whereas the simulation has free motion.
This behavior arises from our dynamics model: for this work , we do not model the dynamics of the environment on the robot's tip, and rely on the assumption that the environment is more compliant than the robot in low-force regimes.
The hardware test shows the assumption was a realistic approximation for safe regions with low contact.
More aggressive motions could be modeled by including the environmental forces per past methods in PCC soft dynamics \cite{dellasantina_model_2020}, though a hybrid system may arise.

\vspace{-0.2cm}
\subsection{Theoretical Implications and Limitations}

The controller in this manuscript enforces safety for the end effector frame of a soft manipulator, but does not address whole-body safety: it is possible that other points along the robot could exit $\mathcal{P}$. 
Our analysis isolates the core theoretical advance, and aligns with prevailing approaches in task-space control which typically prioritize end-effector interaction as the primary site of environmental contact.
Prior work in rigid robots addresses whole-body contact by adding CBFs at other relevant frames along the robot’s kinematics \cite{landi_safety_2019}, which could be adopted for our setting without theoretical changes.
We note that preliminary untested results in soft robotics have considered a similar concept \cite{wong_contact-aware_2025}.

This manuscript formulates, simulates, and tests our method for a pneumatic 2D soft robot only.
However, the PCC dynamics model \cite{dellasantina_model_2020} and formulation of a similar controller on a rigid robot \cite{rauscher_constrained_2016} have both been demonstrated in 3D environments, so we hypothesize that 3D motions are primarily a technical rather than theoretical challenge for future work.
Prior results in self-collision for soft robots have shown validity of the 3D dynamics with CBF constraints \cite{patterson_safe_2024}.
Similarly, other soft robot dynamics which are control-affine may also work with our method, such as cable-driven robots \cite{CaluwaertsDesign2014}, though more theoretical investigation is needed for slackness and hysteresis.

\vspace{-0.2cm}
\section{Conclusion and Future Work}
\label{sec:conclusion}

This work represents, to the authors' knowledge, the first formal closed-loop safety-verified controller for a soft robot manipulator's environmental interactions.
The work here suggests an entirely different perspective on control of soft robots, one which simultaneously avoids the paradoxical issues of precision and tracking control in a soft system, while also giving an engineering analysis to claims of `safety' in soft robotics.
Extending and applying this concept may provide the intelligence needed to bring soft manipulation into the real-world environments. 

The basic control barrier function in this manuscript has limitations on the speed of motion in the robot as well as applicability to more complex scenarios \cite{XiaoHigh-Order2022}.
Future work could implement higher-performance barrier function candidates such as High Order CBFs \cite{patterson_safe_2024,wong_contact-aware_2025}, and compare against this manuscript's method as a baseline. 

This manuscript’s assumptions about the environment surrounding the soft robot are simplified to thoroughly evaluate the implications of the control scheme. 
However, realistic scenarios in human-robot contact involve dynamic environments \cite{landi_safety_2019} that change over time and are often unknown a-priori \cite{wang2024guarding}, as well as non-uninform environments.
Time-varying environments can be readily addressed by using time-varying CBFs for time-varying safety constraints \cite{rauscher_constrained_2016}.
Alternatively, adaptive parameter estimation could be used in fully unknown environments \cite{pati_robust_2023}. 
Similarly, the no-contact set $\mathcal{N}$ could be made to arbitrary geometries by cell decomposition \cite{DicksonSpline2024}, and calculation of the safe set $\mathcal{P}$ from $\mathcal{N}$ could then incorporate localized elasticities or tangential deformation. 
None of these would modify the structure of the controller nor the assumptions about the robot.

Finally, extensions could include formal studies of robustness (e.g., with respect to uncertainty in model or environment) to quantify the relationship between tuning of the CBF and the validity of a safety proof. 
We note that prior work in CBFs for rigid robots has shown robustness in these settings \cite{wang2024guarding}.
Similarly, future work using other safety methods such as reachability \cite{akametalu_reachability-based_2014} could compare performance against the first-of-its-kind, proof-of-concept contributed here.

\bibliographystyle{IEEEtran}
\bibliography{references}

\begin{thebibliography}{10}
\providecommand{\url}[1]{#1}
\csname url@samestyle\endcsname
\providecommand{\newblock}{\relax}
\providecommand{\bibinfo}[2]{#2}
\providecommand{\BIBentrySTDinterwordspacing}{\spaceskip=0pt\relax}
\providecommand{\BIBentryALTinterwordstretchfactor}{4}
\providecommand{\BIBentryALTinterwordspacing}{\spaceskip=\fontdimen2\font plus
\BIBentryALTinterwordstretchfactor\fontdimen3\font minus
  \fontdimen4\font\relax}
\providecommand{\BIBforeignlanguage}[2]{{%
\expandafter\ifx\csname l@#1\endcsname\relax
\typeout{** WARNING: IEEEtran.bst: No hyphenation pattern has been}%
\typeout{** loaded for the language `#1'. Using the pattern for}%
\typeout{** the default language instead.}%
\else
\language=\csname l@#1\endcsname
\fi
#2}}
\providecommand{\BIBdecl}{\relax}
\BIBdecl

\bibitem{majidi_2014_soft_robots}
C.~Majidi, ``Soft robotics: A perspective—current trends and prospects for
  the future,'' \emph{Soft Robotics}, vol.~1, no.~1, pp. 5--11, 2014.

\bibitem{abidi_intrinsic_2017}
H.~Abidi and M.~Cianchetti, ``On intrinsic safety of soft robots,''
  \emph{Frontiers in Robotics and AI}, vol. Volume 4 - 2017, 2017.

\bibitem{enayati2016haptics}
N.~Enayati, E.~De~Momi, and G.~Ferrigno, ``Haptics in robot-assisted surgery:
  Challenges and benefits,'' \emph{IEEE Reviews in Biomedical Engineering},
  vol.~9, pp. 49--65, 2016.

\bibitem{vasicSafetyIssuesHumanrobot2013}
M.~Vasic and A.~Billard, ``Safety issues in human-robot interactions,'' in
  \emph{2013 IEEE International Conference on Robotics and Automation}, 2013,
  pp. 197--204.

\bibitem{li_formal_2019}
X.~Li, Z.~Serlin, G.~Yang, and C.~Belta, ``A formal methods approach to
  interpretable reinforcement learning for robotic planning,'' \emph{Science
  Robotics}, vol.~4, no.~37, p. eaay6276, 2019.

\bibitem{dang_reachability_1998}
T.~Dang and O.~Maler, ``\BIBforeignlanguage{en}{Reachability analysis via face
  lifting},'' in \emph{\BIBforeignlanguage{en}{Hybrid {Systems}: {Computation}
  and {Control}}}, ser. Lecture {Notes} in {Computer} {Science}, T.~A.
  Henzinger and S.~Sastry, Eds.\hskip 1em plus 0.5em minus 0.4em\relax
  Springer, 1998, pp. 96--109.

\bibitem{akametalu_reachability-based_2014}
A.~K. Akametalu, J.~F. Fisac, J.~H. Gillula, S.~Kaynama, M.~N. Zeilinger, and
  C.~J. Tomlin, ``Reachability-based safe learning with gaussian processes,''
  in \emph{53rd IEEE Conference on Decision and Control}, 2014, pp. 1424--1431.

\bibitem{kochdumper_provably_2023}
N.~Kochdumper, H.~Krasowski, X.~Wang, S.~Bak, and M.~Althoff, ``Provably safe
  reinforcement learning via action projection using reachability analysis and
  polynomial zonotopes,'' \emph{IEEE Open Journal of Control Systems}, vol.~2,
  pp. 79--92, 2023.

\bibitem{ames_control_2014}
A.~D. Ames, J.~W. Grizzle, and P.~Tabuada, ``Control barrier function based
  quadratic programs with application to adaptive cruise control,'' in
  \emph{53rd IEEE Conference on Decision and Control}, 2014, pp. 6271--6278.

\bibitem{funada_visual_2019}
R.~Funada, M.~Santos, J.~Yamauchi, T.~Hatanaka, M.~Fujita, and M.~Egerstedt,
  ``Visual coverage control for teams of quadcopters via control barrier
  functions,'' in \emph{2019 International Conference on Robotics and
  Automation (ICRA)}, 2019, pp. 3010--3016.

\bibitem{rauscher_constrained_2016}
M.~Rauscher, M.~Kimmel, and S.~Hirche, ``Constrained robot control using
  control barrier functions,'' in \emph{2016 IEEE/RSJ International Conference
  on Intelligent Robots and Systems (IROS)}, 2016, pp. 279--285.

\bibitem{landi_safety_2019}
C.~T. Landi, F.~Ferraguti, S.~Costi, M.~Bonfè, and C.~Secchi, ``Safety barrier
  functions for human-robot interaction with industrial manipulators,'' in
  \emph{2019 18th European Control Conference (ECC)}, 2019, pp. 2565--2570.

\bibitem{sabelhaus_safe_2024}
A.~P. Sabelhaus, Z.~J. Patterson, A.~T. Wertz, and C.~Majidi, ``Safe
  supervisory control of soft robot actuators,'' \emph{Soft Robotics}, vol.~11,
  no.~4, pp. 561--572, 2024, pMID: 38324015.

\bibitem{anderson_maximizing_2024}
M.~L. Anderson, R.~Jing, J.~C. Pacheco~Garcia, I.~Yang, S.~Alizadeh-Shabdiz,
  C.~DeLorey, and A.~P. Sabelhaus, ``Maximizing consistent high-force output
  for shape memory alloy artificial muscles in soft robots,'' in \emph{2024
  IEEE 7th International Conference on Soft Robotics (RoboSoft)}, 2024, pp.
  213--219.

\bibitem{patterson_safe_2024}
Z.~J. Patterson, W.~Xiao, E.~Sologuren, and D.~Rus, ``Safe control for
  soft-rigid robots with self-contact using control barrier functions,'' in
  \emph{2024 IEEE 7th International Conference on Soft Robotics
  (RoboSoft)}.\hskip 1em plus 0.5em minus 0.4em\relax IEEE, 2024, pp. 151--156.

\bibitem{dyck_impedance_2022}
M.~Dyck, A.~Sachtler, J.~Klodmann, and A.~Albu-Schäffer, ``Impedance control
  on arbitrary surfaces for ultrasound scanning using discrete differential
  geometry,'' \emph{IEEE Robotics and Automation Letters}, vol.~7, no.~3, pp.
  7738--7746, 2022.

\bibitem{bajo_hybrid_2016}
A.~Bajo and N.~Simaan, ``Hybrid motion/force control of multi-backbone
  continuum robots,'' \emph{The International Journal of Robotics Research},
  vol.~35, no.~4, pp. 422--434, 2016.

\bibitem{balasubramanian_faulttolerant_2020}
L.~Balasubramanian, T.~Wray, and D.~D. Damian, ``Fault tolerant control in
  shape-changing internal robots,'' in \emph{2020 IEEE International Conference
  on Robotics and Automation (ICRA)}, 2020, pp. 5502--5508.

\bibitem{feng_safety_2021}
Y.~Feng, T.~Ide, H.~Nabae, G.~Endo, R.~Sakurai, S.~Ohno, and K.~Suzumori,
  ``Safety-enhanced control strategy of a power soft robot driven by hydraulic
  artificial muscles,'' \emph{ROBOMECH Journal}, vol.~8, 03 2021.

\bibitem{garg_autonomous_2025}
S.~Garg, M.~Goharimanesh, S.~Sajjadi, and F.~Janabi-Sharifi, ``Autonomous
  control of soft robots using safe reinforcement learning and covariance
  matrix adaptation,'' \emph{Engineering Applications of Artificial
  Intelligence}, vol. 153, p. 110791, 2025.

\bibitem{li_discrete_2023}
H.~Li, L.~Xun, G.~Zheng, and F.~Renda, ``Discrete cosserat static model-based
  control of soft manipulator,'' \emph{IEEE Robotics and Automation Letters},
  vol.~8, no.~3, pp. 1739--1746, 2023.

\bibitem{della_santina_model-based_2023}
C.~Della~Santina, C.~Duriez, and D.~Rus, ``Model-based control of soft robots:
  A survey of the state of the art and open challenges,'' \emph{IEEE Control
  Systems Magazine}, vol.~43, no.~3, pp. 30--65, 2023.

\bibitem{bruder_data-driven_2020}
D.~Bruder, X.~Fu, R.~B. Gillespie, C.~D. Remy, and R.~Vasudevan, ``Data-driven
  control of soft robots using koopman operator theory,'' \emph{IEEE
  Transactions on Robotics}, vol.~37, no.~3, pp. 948--961, 2021.

\bibitem{haggerty_control_2023}
D.~A. Haggerty, M.~J. Banks, E.~Kamenar, A.~B. Cao, P.~C. Curtis, I.~Mezić,
  and E.~W. Hawkes, ``Control of soft robots with inertial dynamics,''
  \emph{Science Robotics}, vol.~8, no.~81, p. eadd6864, 2023.

\bibitem{thieffry_control_2019}
M.~Thieffry, A.~Kruszewski, C.~Duriez, and T.-M. Guerra, ``Control design for
  soft robots based on reduced-order model,'' \emph{IEEE Robotics and
  Automation Letters}, vol.~4, no.~1, pp. 25--32, 2019.

\bibitem{hachen_non-linear_2025}
M.~Hachen, C.~Shentu, S.~Lilge, and J.~Burgner-Kahrs, ``A non-linear model
  predictive task-space controller satisfying shape constraints for
  tendon-driven continuum robots,'' \emph{IEEE Robotics and Automation
  Letters}, vol.~10, no.~3, pp. 2438--2445, 2025.

\bibitem{ames_control_2019}
A.~D. Ames, S.~Coogan, M.~Egerstedt, G.~Notomista, K.~Sreenath, and P.~Tabuada,
  ``Control barrier functions: Theory and applications,'' in \emph{2019 18th
  European Control Conference (ECC)}, 2019, pp. 3420--3431.

\bibitem{dellasantina_model_2020}
C.~D. Santina, R.~K. Katzschmann, A.~Bicchi, and D.~Rus, ``Model-based dynamic
  feedback control of a planar soft robot: trajectory tracking and interaction
  with the environment,'' \emph{The International Journal of Robotics
  Research}, vol.~39, no.~4, pp. 490--513, 2020.

\bibitem{BamdadKinematics2019}
M.~Bamdad and M.~M. Bahri, ``Kinematics and manipulability analysis of a highly
  articulated soft robotic manipulator,'' \emph{Robotica}, vol.~37, no.~5, p.
  868–882, 2019.

\bibitem{stolzle_input--state_2024}
M.~Stölzle and C.~D. Santina, ``Input-to-state stable coupled oscillator
  networks for closed-form model-based control in latent space,'' in
  \emph{Proceedings of the 38th International Conference on Neural Information
  Processing Systems}, ser. NIPS '24.\hskip 1em plus 0.5em minus 0.4em\relax
  Curran Associates Inc., 2025.

\bibitem{patterson_robust_2022}
Z.~J. Patterson, A.~P. Sabelhaus, and C.~Majidi, ``Robust control of a
  multi-axis shape memory alloy-driven soft manipulator,'' \emph{IEEE Robotics
  and Automation Letters}, vol.~7, no.~2, pp. 2210--2217, 2022.

\bibitem{della_santina_controlling_2017}
C.~Della~Santina, M.~Bianchi, G.~Grioli, F.~Angelini, M.~Catalano, M.~Garabini,
  and A.~Bicchi, ``Controlling soft robots: Balancing feedback and feedforward
  elements,'' \emph{IEEE Robotics \& Automation Magazine}, vol.~24, no.~3, pp.
  75--83, 2017.

\bibitem{thuruthel_manipulator_2019}
T.~G. Thuruthel, E.~Falotico, F.~Renda, and C.~Laschi, ``Model-based
  reinforcement learning for closed-loop dynamic control of soft robotic
  manipulators,'' \emph{IEEE Transactions on Robotics}, vol.~35, no.~1, pp.
  124--134, 2019.

\bibitem{pereira_challenges_2020}
A.~Pereira and C.~Thomas, ``Challenges of machine learning applied to
  safety-critical cyber-physical systems,'' \emph{Machine Learning and
  Knowledge Extraction}, vol.~2, no.~4, pp. 579--602, 2020.

\bibitem{DellaSantina_OnAnImproved_2020}
C.~Della~Santina, A.~Bicchi, and D.~Rus, ``On an improved state parametrization
  for soft robots with piecewise constant curvature and its use in model based
  control,'' \emph{IEEE Robotics and Automation Letters}, vol.~5, no.~2, pp.
  1001--1008, 2020.

\bibitem{liu_online_2021}
S.~B. Liu and M.~Althoff, ``Online verification of impact-force-limiting
  control for physical human-robot interaction,'' in \emph{2021 IEEE/RSJ
  International Conference on Intelligent Robots and Systems (IROS)}, 2021, pp.
  777--783.

\bibitem{fujikawa_critical_2017}
T.~Fujikawa, R.~Sugiura, R.~Nishikata, and T.~Nishimoto, ``Critical contact
  pressure and transferred energy for soft tissue injury by blunt impact in
  human-robot interaction,'' in \emph{2017 17th International Conference on
  Control, Automation and Systems (ICCAS)}, 2017, pp. 867--872.

\bibitem{dickson_real-time_2025}
A.~Dickson, J.~C.~P. Garcia, R.~Jing, M.~L. Anderson, and A.~P. Sabelhaus,
  ``Real-{Time} {Trajectory} {Generation} for {Soft} {Robot} {Manipulators}
  {Using} {Differential} {Flatness},'' in \emph{{IEEE} {International}
  {Conference} on {Soft} {Robotics} ({RoboSoft})}, Apr. 2025.

\bibitem{Hsu2015Control}
S.-C. Hsu, X.~Xu, and A.~D. Ames, ``Control barrier function based quadratic
  programs with application to bipedal robotic walking,'' in \emph{2015
  American Control Conference (ACC)}, 2015, pp. 4542--4548.

\bibitem{wang_apriltag_2016}
J.~Wang and E.~Olson, ``Apriltag 2: Efficient and robust fiducial detection,''
  in \emph{2016 IEEE/RSJ International Conference on Intelligent Robots and
  Systems (IROS)}, 2016, pp. 4193--4198.

\bibitem{wong_contact-aware_2025}
K.~Wong, M.~Stoelzle, W.~Xiao, C.~Della~Santina, D.~Rus, and G.~Zardini,
  ``Contact-aware safety in soft robots using high-order control barrier and
  lyapunov functions,'' \emph{arXiv:2505.03841}, May 2025.

\bibitem{CaluwaertsDesign2014}
K.~Caluwaerts, J.~Despraz, A.~Işçen, A.~P. Sabelhaus, J.~Bruce, B.~Schrauwen,
  and V.~SunSpiral, ``Design and control of compliant tensegrity robots through
  simulation and hardware validation,'' \emph{Journal of The Royal Society
  Interface}, vol.~11, no.~98, p. 20140520, 2014.

\bibitem{XiaoHigh-Order2022}
W.~Xiao and C.~Belta, ``High-order control barrier functions,'' \emph{IEEE
  Transactions on Automatic Control}, vol.~67, no.~7, pp. 3655--3662, 2022.

\bibitem{wang2024guarding}
X.~Wang, J.~Yang, J.~Mao, J.~Liang, S.~Li, and Y.~Yan, ``Guarding force:
  Safety-critical compliant control for robot-environment interaction,''
  \emph{arXiv:2405.04859}, May 2024.

\bibitem{pati_robust_2023}
T.~Pati and S.~Z. Yong, ``Robust control barrier functions for control affine
  systems with time-varying parametric uncertainties,''
  \emph{IFAC-PapersOnLine}, vol.~56, no.~2, pp. 10\,588--10\,594, 2023, 22nd
  IFAC World Congress.

\bibitem{DicksonSpline2024}
A.~Dickson, C.~G. Cassandras, and R.~Tron, ``Spline trajectory tracking and
  obstacle avoidance for mobile agents via convex optimization,'' in \emph{2024
  IEEE 63rd Conference on Decision and Control (CDC)}, 2024, pp. 6572--6577.

\end{thebibliography}

\end{document}